%% file: ecai2020.tex
\documentclass{utils/ecai}
\usepackage{mathtools} \usepackage{amsmath} 
\usepackage{amssymb}
\usepackage{times}
\usepackage{graphicx}
\usepackage{latexsym}
\usepackage{comment} \usepackage{amsthm} \usepackage{hyperref}
\usepackage[nolist]{acronym} \usepackage{array}
\usepackage{float} \usepackage{siunitx}

\graphicspath{ {utils/images/} }
\ecaisubmission 
\input{utils/acro_list}

\usepackage{color}

\newtheorem{definition}{Definition}
\newtheorem{thm}[theorem]{Theorem}
\newtheorem{cor}{Corollary}
\newtheorem{rem}{Remark}
\newtheorem{prop}{Proposition}

\newtheorem{example}{Example}
\newcolumntype{P}[1]{>{\centering\arraybackslash}p{#1}}
\begin{document}

\title{Extended Markov Games to Learn  Multiple Tasks \\ in Multi-Agent Reinforcement Learning}

\author{Borja G. Le\'on \and Francesco Belardinelli \institute{Imperial College London,
UK, email: bg19@imperial.ac.uk, francesco.belardinelli@imperial.ac.uk} }

\maketitle

\begin{abstract}
    The combination of Formal Methods with Reinforcement Learning (RL) has recently
    attracted interest as a way for single-agent RL to learn multiple-task specifications.
    In this paper we extend this convergence to multi-agent settings and formally define
    Extended Markov Games as a general mathematical model that allows multiple RL agents
    to concurrently learn various non-Markovian specifications. To introduce this new
    model we provide formal definitions and proofs as well as empirical tests of RL
    algorithms running on this framework. Specifically, we use our model to train two
    different logic-based multi-agent RL algorithms to solve diverse settings of
    non-Markovian {\em co-safe} $LTL$ specifications.

\end{abstract}

\section{Introduction} \label{sec:intro}

\input{introduction.tex}

\section{Background} \label{sec:backg}
\input{background.tex}

\section{Non-Marvokian Multi-Agent Specifications}
\input{EMG.tex}

\section{Deep MARL with {\em co-safe }$LTL$ Goals}  \label{sec:MARL}
\input{Deep_MARL.tex}

\section{Experimental Evaluation}\label{sec:experiments}

\input{experiments.tex}

\section{Conclusions} \label{conclusions}
\input{conclusions.tex}

\bibliographystyle{utils/ecai}
\bibliography{utils/Bib}
\end{document}

%% file: utils/acro_list.tex
\begin{acronym}[rPATL] 
\acro{PATL}{Probabilistic Alternating-time Temporal Logic}
\acro{rPATL}{PATL with rewards}
\acro{RL}{Reinforcement Learning}
\acro{MARL}{Multi-Agent RL}
\acro{IRL}{Inverse RL}
\acro{DRL}{Deep Reinforcement Learning}
\acro{MC}{Model Checking}
\acro{PCTL}{Probabilistic Branching-time Temporal Logic}
\acro{TL}{Temporal Logic}
\acro{MDP}{Markov Decision Process}
\acro{pCGS}{probabilistic Concurrent Game Structure}
\acro{MG}{Markov Game}
\acro{EMG}{Extended Markov Game}
\acro{FM}{Formal Methods}
\acro{DL}{Deep Learning}
\acro{LTL}{Linear-time Temporal Logic}
\acro{STL}{Signal Temporal Logic}
\acro{DQNs}{Deep Q-Networks}
\acro{HRL}{Hierarchical Reinforcement Learning}
\acro{DFA}{Deterministic Finite Automaton}
\acro{NMRDP}{Non-Markovian Reward Decision Process}
\acro{MTCLS}{multi-task co-safe LTL specification}
\acro{RM}{Reward Machine}
\acro{NMRG}{Non-Markovian Reward Game}
\acro{SMAC}{StarCraft Multi-Agent Challenge}
\acro{EAs}{Evolutionary Algorithms}
\acro{IQL}{Independent Q-Learning}
\acro{AI}{Artificial Intelligence}
\acro{STL}{Signal Temporal Logic}
\acro{I-DQN}{Independent Deep Q-Learning}
\acro{LPOPL}{LTL Progressionfor Off-Policy Learning}
\acrodefplural{MG}[MGs]{Markov Games}
\end{acronym}

%% file: introduction.tex
Reinforcement Learning is increasingly becoming a dominant paradigm
for training autonomous agents in unknown and high-dimensional
scenarios \cite{mnih2015human,schulman2017proximal}. Its combination
with temporal logic from Formal Methods
\cite{de2018reinforcement,
toro2018teaching}, holds considerable promise to help learning agents
tackle several tasks as a single specification, including safety
constrains, which play a major role in many real-world scenarios, such
as autonomous driving \cite{shalev2016safe}, autonomous
robots \cite{junges2016safety} or network packet delivery \cite{ye2015multi}. 
Additionally, temporal logic offers a simple but efficient way to
use \ac{RL} algorithms to find optimal policies in
non-markovian settings \cite{bacchus1996rewarding, brafman2018ltlf,
de2018reinforcement}.

These real-world scenarios, in which agents have to work in an 
environment where there are other learning agents, are naturally
modelled as a multi-agent system (MAS). Unfortunately, MAS have
several challenges on their own, such as scalability with the number
of agents as well as issues pertaining to non-stationarity, i.e, that 
in MAS probabilities attached to transitions do not depend solely on the
single agent's actions and the corresponding transition probabilities.

This dependency on other agents' policies means that the notion
of optimality, central to \acp{MDP}, needs
to be encapsulated in a sort of equilibrium to solve the optimization
problem. In this context, Nash equilibrium \cite{hu1998multiagent}, 
where each agent's policy is the best response to the others', 
is one of the most common solution concepts. \ac{MARL} has the
potential to overcome scalability shortcomings by showing promising
results in applications such as robot
swarms \cite{huttenrauch2017guided} and the aforementioned autonomous
vehicles \cite{cao2012overview}.

Despite the growing success of \ac{MARL}, also due to the adoption of
neural networks in Deep (RL) \cite{tampuu2015multiagent}, there is still 
little work on combining these multi-agent learning techniques with formal methods.
This is the long-term challenge that we here start to address.

\textbf{Contribution.}
This paper focus on formally extending \acp{MG}, the mathematical
model that is traditionally used in MARL, to build a new general model, i.e, not focused solely in one kind of multi-agent game, that allows multiple learning agents to concurrently fulfill various
non-Markovian specifications in multi-agent settings. To support our
model with empirical evidence, we also extended two logic-based RL
algorithms to multi-agents systems in order to show how various
learning agents can fulfill different types of non-Markovian
specifications expressed in {\em co-safe-} Linear-time Temporal Logic
($LTL$).  Our results are promising and point to interesting future
work we discuss in Sec.~\ref{conclusions}.

\textbf{Related Work}. There is a growing interest
towards the interactions of Formal Methods and Reinforcement
Learning.
In \cite{alshiekh2018safe} model checking of temporal logics is
adopted to build a shielding system that is meant to prevent a
learning agent from taking dangerous
actions. In \cite{mason2017assured} verification
techniques are applied to abstract MDPs to check abstracted
policies learned with RL. Related works (see, e.g., \cite{zhou2018safety,
wen2017learning})
pursuit similar verification goals in the sub-field of Inverse
Reinforcement Learning, where reward functions are initially unknown and 
Formal Methods enable some degree of safety on the extracted rewards.
Closer to the present contribution,
\cite{muniraj2018enforcing}
tackles MARL with rewards expressed in temporal logic. Specifically,
 they enforce \ac{STL} specifications in a mini-max game
 extension of the popular decentralized \ac{I-DQN} algorithm for
 adversarial settings. Our approach differs under two key aspects.
 First, we provide a formal definition on how to use \ac{FM} in order
 train \ac{MARL} policies to solve \ac{NMRG},
 which is not restricted to a specific temporal logic, such as \ac{STL}. Second,
 our approach is meant to cover a wide range of multi-agent scenarios, not only
 adversarial as in \cite{muniraj2018enforcing}.

Solving multiple task-specifications expressed in temporal logic has
also been previously addressed in \cite{de2018reinforcement}, where
authors show how an RL agent is able to learn non-Markovian rewards
expressed in $LTL$ over finite traces ($LTL_f$) and  Linear
Dynamic Logic over finite traces ($LDL_f$), by using reward
shaping \cite{ng1999policy}. Toro et al. \cite{toro2018teaching} also
tackle the problem of learning multiple tasks expressed in {\em co-safe}
$LTL$.
They introduce an extension of Q-learning called \ac{LPOPL} that makes
use of the additional information provided by the temporal logic
specifications in order to achieve an agent's goal. Our work focus on
extending what \cite{de2018reinforcement, toro2018teaching} developed
for the single-agent framework to multi-agent systems, where
reinforcement learning guided by temporal specifications could prove
to be an effective tool to address some of the complex challenges
mentioned above, including scalability, stationarity or equilibrium of
solutions.


%% file: background.tex
In this section we recall preliminary notions on Markov games,
the linear-time temporal logic $LTL$, and Q-learning.

\paragraph{Markov Games.}\label{subsec:MG}

A single-agent \ac{RL} task is typically modelled as a \ac{MDP}, which
means that, in order to work with the performing algorithms from RL,
the reward given to the learning agent need to be Markovian, i.e.,
it must depend only on the last state and action taken.

A multi-agent \ac{RL} task, i.e, a problem-solving cooperative,
competitive or adversarial setting where there are two or more
learning agents whose goal is to individually collect the maximum
possible reward by interacting with the environment, is modeled as a
Markov game \cite{littman1994markov}.

\begin{definition}[MG] \label{markovgame}
A {\em Markov game} is a tuple $\mathcal{G}=\langle N, S, A, P, R,
\gamma\rangle$ where:
\begin{itemize}
\item $N$ is the \textit{set of $n$ agents}.
\item $S$ is a finite set of \textit{states}.
\item $A$ is the \textit{action set} where each $A_{i}$
    is the action set of agent $i \in N$. $A^t$ refers to the joint action taken by all agents in state $s_t$.

  \item $P: S \times A_{1} \times \cdots \times A_{n} \times S \rightarrow[0,1]$
    is the \textit{transition probability function} that returns the probability of transitioning to a new state given the previous state and the actions taken by all agents.

  \item $R_i : S \times A_1 \times \ldots \times A_N \rightarrow \mathbb{R}$ is the reward function of agent $i$ and $R=\{R_1, \ldots ,R_N\}$  is the set of reward functions.
\item $\gamma \in(0,1]$ is the \textit{discount factor} that is used to
    discount future rewards.
\end{itemize} 
\end{definition}

Markov games can be thought of as multi-agent extension of MDPs where
actions are chosen and executed simultaneously, thus new states depend
on the action taken by each single agent. Notice also that we are
working in a fully observable setting, which means that
individual observation functions are trivial, and therefore omitted
for sake of simplicity.\\

\noindent
{\bf Linear Time Temporal Logic}
\label{subsec:LTL}
\cite{pnueli1977temporal} is an expressive and well-studied
logic-based language to specify
  properties of MAS modelled as Markov games. Formulae in $LTL$
are built from a set $\mathcal{AP}$ of atoms,
by using Boolean and temporal operators as follows:
\begin{equation}
\varphi \enskip::=\enskip p\mid\neg \varphi\mid \varphi_{1} \wedge \varphi_{2}\mid \bigcirc \varphi\mid \varphi_{1} \mathrm{U} \varphi_{2}
\end{equation}
where $p \in \mathcal{AP}$, $\bigcirc$ is the {\em next} operator, and $\mathrm{U}$ is the {\em until}
operator. We use the standard abbreviations:
{\em eventually}  $\diamond \varphi \,\equiv \top
\mathrm{U} \varphi$; {\em always} $\square \varphi \,\equiv\, \neg \diamond
\neg \varphi$;
{\em weak next} $\bullet \varphi
\,\equiv\, \neg \bigcirc \neg \varphi$ (when interpreted over finite traces $\neg
\bigcirc \varphi \not\equiv \bigcirc \neg \varphi$); and $Last
\,\equiv\, \bullet false$ denotes the end of the trace (if finite).

The semantics of $LTL$ formulae over $\mathcal{AP}$ are defined over
(finite or infinite) sequences $\sigma=\left\langle\sigma_{0}, \sigma_{1}, \sigma_{2},
\ldots\right\rangle \in (2^{\mathcal{AP}})^{\omega} \cup (2^{\mathcal{AP}})^{+}$ of truth assignments
 for the atoms in $\mathcal{AP}$, where each $\sigma_i$ is a truth
assignment to each atom in $\mathcal{AP}$. By $p
\in \sigma_{i}$, for atom $p \in \mathcal{AP}$, we mean that $p$ is
true in $\sigma_i$.  The length of a sequence $\sigma$ is denoted as
$|\sigma|$, with $|\sigma| = \omega$ if $\sigma$ is infinite.
For $i \leq |\sigma|$, let $\sigma_{\geq i}$ be the suffix
$\sigma_{i}, \sigma_{i+1}, \ldots$ of $\sigma$ starting at $\sigma_{i}$ and $\sigma_{\leq i}$ its prefix
$\sigma_{0},\ldots, \sigma_{i}$.


We can now define formally when a sequence
$\sigma$ \textit{satisfies} an $LTL$ formula $\varphi$ at time $i\geq 0$,
denoted by $\langle\sigma, i\rangle \models \varphi$, as follows
(clauses for Boolean connective are immediate and thus omitted):
\begin{tabbing}
  $\langle\sigma, i\rangle \models p$ \ \ \ \ \ \ \  \ \ \ \ \= iff \= $p \in \sigma_{i}$\\
$\langle\sigma, i\rangle \models \bigcirc \varphi$ \> iff \> $|\sigma| \geq i+1$ and $\langle\sigma, i+1\rangle \models \varphi$\\
  $\langle\sigma, i\rangle \models \varphi_{1} \mathrm{U} \varphi_{2}$ \> iff \> for some $j$, $i \leq j \leq |\sigma|$ and $\langle\sigma, j\rangle \models \varphi_{2}$, and\\
  \> \> for all $k$, $i \leq k<j$ implies $\langle\sigma, k\rangle \models \varphi_{1}$
\end{tabbing}

A sequence $\sigma$ is said to satisfy or model $\varphi$ iff $\langle
\sigma, 0 \rangle \models \varphi$. Notice that, whenever $\sigma$ is
infinite, we obtain the standard semantics for $LTL$.

We mentioned in Sec.~\ref{sec:intro} two variants of $LTL$ that
have been used in the latest works extending single-agent RL with temporal
logic \cite{de2018reinforcement, toro2018teaching}. Both are motivated by 
the fact that, in RL, agents are typically trained over
finite episodes, but in different way. The former is
{\em $LTL$ over finite traces} ($LTL_f$), which is standard $LTL$
interpreted over finite traces only, instead of infinite
ones \cite{de2013linear}. The latter is {\em
co-safe} $LTL$ \cite{kupferman2001model}, which -- like $LTL$ -- is interpreted
over infinite traces, but 
its syntax is restricted
as follows:
\begin{equation}
  \varphi \enskip::=\enskip p \mid\neg p \mid \varphi_{1} \wedge \varphi_{2} \mid \varphi_{1} \vee \varphi_{2} \mid \bigcirc \varphi\mid \varphi_{1} \mathrm{U} \varphi_{2}
\end{equation}

The restricted syntax of {\em co-safe} $LTL$ is meant to guarantee
that formulas, if true, are true 
after a finite number of steps.
We can show that $LTL_f$ and {\em co-safe} $LTL$ are actually
equivalent, with the former restricting the semantics while the later
restricts the syntactic. The equivalence is stated in
Corollary~\ref{cor1} below.

First, by next result every {\em co-safe} $LTL$ formula is true iff it
is satisfied by a finite sequence.
\begin{prop} \label{eq:co-safe LTL_f}
  Let $\sigma \in (2^{\mathcal{AP}})^{\omega}$ be an infinite sequence, and  $\varphi$ a  {\em co-safe} $LTL$ formula. Then,
\begin{eqnarray} 
\sigma \models \varphi & \text{iff} & \text{for some $i \geq 0$}, \sigma_{\leq i} \models \varphi
\end{eqnarray}
\end{prop}

\begin{proof}
  We first prove direction $\Rightarrow$ by induction on the structure of $\varphi$.
%
Suppose $\sigma \models
\varphi$. We compute the value $i$ by induction on the structure of
$\varphi$ as follows:
\begin{tabbing}
  $i(\sigma, p)$ \ \ \ \ \ \ \ \ \ \ \ \ \= $=$ \= $0 = i(\sigma, \neg p)$\\
  $i(\sigma, \varphi_1 \land \varphi_2)$ \> $=$ \> $\max\{ i(\sigma,
  \varphi_1),i(\sigma,   \varphi_2)\}$\\
  $i(\sigma, \varphi_1 \lor \varphi_2)$ \> $=$ \> $\min\{ i(\sigma,
  \varphi_1),i(\sigma,   \varphi_2)\}$\\
  $i(\sigma, \bigcirc
  \varphi_1)$ \> $=$ \> $i(\sigma_{\geq 1}, \varphi_1) +1$ \\
  $i(\sigma,
  \varphi_1 \mathrm{U} \varphi_2)$ \> $=$ \> $\max\{ \{i(\sigma_{\geq k},
  \varphi_1)\}_{k < j},i(\sigma_{\geq j},
  \varphi_2)\} + j$,\\
\> \>
  where $j$ is the smallest $\geq 0$ s.t.~$\langle\sigma, j\rangle \models
  \varphi_{2}$, and\\ \> \> for all $k$, $i \leq k<j$ implies
  $\langle\sigma, k\rangle \models \varphi_{1}$.
\end{tabbing}

  If $\varphi = p \in AP$, then $i(\sigma, \varphi) = 0$ and clearly
  $\sigma_{0} = \sigma_{\leq 0} \models \varphi$, as
  $p \in \sigma_{0}$ by hypothesis. The base case for $\varphi = \neg p$ is
  similar, and the inductive cases for Boolean connectives are immediate.

If $\varphi = \bigcirc \varphi_1$, then $i(\sigma, \varphi) = i(\sigma_{\geq
  1}, \varphi_1) + 1$ and by induction hypothesis $\sigma_{\leq
  i(\sigma_{\geq 1}, \varphi_1)} \models \varphi_1$. Hence, for $i(\sigma, \varphi) = i(\sigma_{\geq 1}, \varphi_1) + 1$, we have
$\sigma_{\leq i(\sigma, \varphi)} \models \varphi$.
  
If $\varphi = \varphi_1 \mathrm{U}  \varphi_2$, then $i(\sigma,
  \varphi_1 \mathrm{U} \varphi_2) = \max\{ \{i(\sigma_{\geq k},
  \varphi_1)\}_{k < j},i(\sigma_{\geq j},
  \varphi_2)\} + j$
  and by induction hypothesis $\sigma_{\leq
     i(\sigma_{\geq j}, \varphi_2)} \models \varphi_2$ and for all $k < j$, $\sigma_{\leq
     i(\sigma_{\geq k}, \varphi_1)} \models \varphi_1$.
  Hence, for $i(\sigma, \varphi) = \max\{ \{i(\sigma_{\geq k},
  \varphi_1)\}_{k < j},i(\sigma_{\geq j},
  \varphi_2)\} + j$, we have
$\sigma_{\leq i(\sigma, \varphi)} \models \varphi$.

  As for the $\Leftarrow$ direction, suppose that $\sigma \not \models \varphi$.
  Then, by induction on the structure of $\varphi$ we can show that for every $i \geq 0$, $\sigma_{\leq i} \not \models \varphi$.
\end{proof}

Next we prove that formulae in $LTL_f$ can be translated into {\em
co-safe} $LTL$ in a truth-preserving manner. First of all, we assume
operators $\square$ and $\bullet$ as primitive and consider negation
on atoms only.  We also assume atom $Last$ as primitive, only true in
the last element in $\sigma$.  Then, consider translation $\tau$ from
$LTL$ into {\em co-safe }$LTL$, which commutes with all operators, and such that
\begin{eqnarray*}
\tau(\square \varphi) & = &  \tau(\varphi) \mathrm{U} (Last \land \tau(\varphi))\\
\tau(\bullet \varphi) & = &  \neg Last \to \bigcirc \tau(\varphi)
\end{eqnarray*}
\begin{prop} \label{eq:co-safe LTL_f_2}
  Let $\sigma \in (2^{\mathcal{AP}})^{+}$ be a finite sequence,
  $\varphi$ an $LTL$ formula,
\begin{eqnarray} 
\sigma \models \varphi & \text{iff} & \text{for every $\sigma' \in (2^{\mathcal{AP}})^{\omega}$}, \sigma \cdot \sigma' \models \tau(\varphi)
\end{eqnarray}
where
$\cdot$ is string concatenation.
\end{prop}
\begin{proof}
The proof is by induction on the structure of $\varphi$. The case of
(negated) atoms, as well as Boolean connectives, is immediate.

For $\varphi = \square \psi$, $\sigma \models \varphi$ iff for every
$i \leq |\sigma|$, $\langle \sigma, i \rangle \models \psi$, iff for
every $i \leq |\sigma|$, $\langle \sigma \cdot \sigma',
i \rangle \models \tau(\psi)$ for every $\sigma' \in
(2^{\mathcal{AP}})^{\omega}$, by induction hypothesis, iff
$\sigma \cdot \sigma' \models \tau(\psi) \mathrm{U}
(Last \land \tau(\psi))$ for every $\sigma' \in
(2^{\mathcal{AP}})^{\omega}$, iff
$\sigma \cdot \sigma' \models \tau(\varphi)$.

For $\varphi = \bullet \psi$, $\sigma \models \varphi$ iff $|\sigma| \geq 1$ implies $\langle \sigma, 1 \rangle \models \psi$, iff for every $\sigma' \in (2^{\mathcal{AP}})^{\omega}$,
$\sigma  \cdot \sigma' \not \models Last$ implies $\sigma \cdot \sigma' \models \bigcirc \tau(\psi)$ by induction hypothesis, iff $\sigma \cdot \sigma' \models  \neg Last \to \bigcirc \tau(\psi)$, iff
$\sigma \cdot \sigma' \models \tau(\varphi)$.
\end{proof}

By combining Proposition~\ref{eq:co-safe LTL_f}
and \ref{eq:co-safe LTL_f_2}, we obtain the following result.
\begin{cor} \label{cor1}
There are truth-preserving (polynomial) translations between $LTL_f$ and {\em co-safe} $LTL$.
\end{cor}



In this paper we work with specifications expressed in {\em
  co-safe }$LTL$,
algorithm \ac{LPOPL} introduced in \cite{toro2018teaching}, which is
based on {\em co-safe }$LTL$. Nevertheless, as can be inferred from 
Sec.~\ref{subsec:EMG}, this approach is not restricted to $LTL$ and 
its variants, but applies to any
formal language whose specifications can be transformed into
a \ac{DFA}, that is, a finite-state machine that accepts or rejects
finite strings of symbols \cite{rabin1959finite}.
\begin{definition}[DFA]
A {\em deterministic finite automaton} is a tuple $\mathcal{A}
= \left\langle \Sigma, Q, q, \delta, F \right\rangle$, where $\Sigma$
is the {\em input alphabet}, $Q$ is the set of {\em states} with {\em
initial state} $q_{0} \in Q$, $\delta: \Sigma \times Q \to Q$ is the transition function, and
$F \subseteq Q$ is the set of {\em final states}.
\end{definition}

For the mathematical framework presented in Sec.~\ref{subsec:EMG} 
to work, we will need to transform {\em co-safe} $LTL$ specifications 
to a correspondent DFA. From \cite{lacerda2015optimal} 
we know that for any co-safe formula
$\varphi_k$ we can build a correspondent DFA
$\mathcal{A}_{\varphi}=\left\langle 2^{\mathcal{AP}}, Q_{\varphi}, q_{\varphi
0}, \delta_{\varphi}, F_{\varphi} \right\rangle$ that accepts exactly the finite
traces that satisfy $\varphi$.

From here on, {\em co-safe} $LTL$ specifications will be referred just as 
$LTL$ specifications without abuse of notation since we proved that both
$LTL_f$ and {\em co-safe} $LTL$ are actually equivalent.
\\

\noindent
\textbf{LTL progression.} $LTL$ formulae can be progressed \cite{bacchus2000using} along a sequence of truth assignments. In \ac{MARL}, that means the formulae representing the specifications to be learned by the agents can be updated during an episode to reflect those requirements from the formulae that have been satisfied by the current history of states. Thus, progressed formulae include only those parts of the original formulae that remain to be satisfied. For example, given the formula $\diamond (p \wedge\diamond q)$ (eventually p and then eventually q) can be progressed to $\diamond q$ once the agents reaches a state where $p$ is true. We now introduce a formal definition of progression similarly to \cite{bacchus2000using}.
\begin{definition}
Given an $LTL$ formula $\varphi$ and a truth assignment $\sigma_k$ over $\mathcal{AP}$,
$prog(\sigma_k,\varphi)$ is defined as follows:
\begin{itemize}{\small
    \item $\operatorname{prog}\left(\sigma_{k}, p\right)=$ true if $p \in \sigma_{k},$ where $p \in \mathcal{AP}$
    \item $\operatorname{prog}\left(\sigma_{k}, p\right)=$ false if $p \notin \sigma_{k},$ where $p \in \mathcal{AP}$
    \item $\operatorname{prog}\left(\sigma_{k}, \neg \varphi\right)=\neg \operatorname{prog}\left(\sigma_{k}, \varphi\right)$
    \item $\operatorname{prog}\left(\sigma_{k}, \varphi_{1} \wedge \varphi_{2}\right)=\operatorname{prog}\left(\sigma_{k}, \varphi_{1}\right) \wedge \operatorname{prog}\left(\sigma_{k}, \varphi_{2}\right)$
    \item $\operatorname{prog}\left(\sigma_{k}, \bigcirc \varphi\right)=\varphi$
    
    \item $\operatorname{prog}(\sigma_{k}, \varphi_{1} \cup \varphi_{2})=\operatorname{prog}(\sigma_{i}, \varphi_{2}) \vee(\operatorname{prog}(\sigma_{k}, \varphi_{1}) \wedge \varphi_{1} \cup \varphi_{2})$}
    
\end{itemize}
\end{definition}

\begin{example} \label{examp:MG}
  As running example throughout the paper, as well as for the
  experimental evaluation in Sec.~\ref{sec:experiments}, we
  consider a Minecraft-like grid world similar to the one introduced
  in \cite{andreas2017modular} but extended with multiple learning
  agents. In this multi-agent scenario, the learning agents can
  interact with objects, extract raw materials from their environment
  and use them to manufacture new objects. This scenario also includes
  a set of features and events that are detectable by the
  agents: \textit{\{got\_wood, used\_toolshed, used\_workbench,
    got\_grass, used\_factory, got\_iron, used\_bridge, used\_axe,
    at\_shelter\}} with obvious interpretations. This set becomes
    the set of atoms $\mathcal{AP}$ in our example. By using these atoms 
    we can specify long-term goals in {\em co-safe} $LTL$, e.g., the specification
    of making shears can be expressed as:
\begin{tabbing}
  $\varphi_{\text
  {shears}}$ \= $\triangleq$ \= $\diamond\left(\text {got}_{-} \text {wood}
\wedge \diamond \text {used}_{-} \text {workbech}\right) \wedge $\\
\> \> $\diamond\left(\text {got}_{-} \text {iron} \wedge \diamond \text
             {used}\_ \text {workbech}\right)$
\end{tabbing}
Notice that wood and iron can be collected either way and that
one workbench usage is enough to fulfill the task as long
as it is done after collecting both items. We call this kind of
specification an interleaving specification.
We suppose that
agents have to collaboratively fulfill a number of multi-task
specifications, such as make shears by collecting wood and iron and using the workbench.
We can model such a scenario as
a Markov game, by defining $S$ as a grid map representing the
positions of the agents in $N$ and the objects within the grid.
Then, the actions in $A_i$ available to each agent $i \in N$ will be the set
$\{Up, Down, Left, Right, Wait \}$. Further, $P$ will represent the probability of
reaching a new state, i.e, a new distribution on the grid map, given
the previous distribution and the joint action taken by the agents.
The reward $R_i$ will be the same for all agents and a positive reward
would be given by the system when the relevant specification is
satisfied. Note that, in order for this model to be Markovian, the
satisfaction of the specifications must depend only on the last state
and action taken.  Finally, the discount factor $\gamma$ will have a value
between 0 or 1, depending on how important we want future rewards to be
for the agents.
\end{example}

\noindent
{\bf Q-learning} \label{subsec:q-learning} \cite{christopher1992watkins} is an off-policy
model-free RL algorithm that is at the core of all the multi-agent
methods we use in our experiments in Sec.~\ref{sec:experiments}. By
off-policy we mean that it is a learning method that aims for a target
policy while using a behavior policy, and by model-free that the
algorithm does not
build a model of the environment to find
the optimal policy, where a
{\em policy} is a mapping
from states to actions.

Q-learning works by initializing the Q-values of every
state-action pair to any value (usually zero). Then, at every
time-step, the algorithm uses a behaviour policy to execute an action
$a$ in the current state $s$, which leads to a new state $s'$ and
reward $r$ returned by the environment. Given a learning rate
$\alpha$ and a future-reward discount factor $\gamma$, the estimated
Q-value $Q(s, a)$ is then updated as follows:
\begin{equation}\label{eq:q-learning}
Q(s, a) \leftarrow Q(s, a)+\alpha\left[r+\gamma \max _{a^{\prime}} Q\left(s^{\prime}, a^{\prime}\right)-Q(s, a)\right]
\end{equation}

Intuitively, Equation \ref{eq:q-learning} means that Q-values, i.e.,
the ``quality'' of a given state-action pair, are updated according to the
weighted sum of their current values and the difference between the
new observed reward outcome plus its expected value and the expected
value given the original state and action taken.

This algorithm is guaranteed to converge to the optimal Q-values as
long as every state-action pair is visited infinitely often \cite{christopher1992watkins}.
A method to fulfill this requirement is to set the behavioral policy to be
$\epsilon$-greedy on the target policy. That is, for each time-step,
the behavior policy selects a random action with probability
$\epsilon$ and the action with the highest Q-value (the one given by
the target policy) with probability $1-\epsilon$.

{\bf Deep Q-learning.}
A popular technique to address the Q-learning algorithm
is by using a table that stores every single state-action pair. This,
however, is impractical for environments with large (or infinite) state
spaces. Deep Q-learning instead makes use of experience replay and
deep neural networks parameterised by $\theta$, where $\theta$ are 
the weights of the neural network to represent an action-value function 
for a given state, that is, by
using neural networks as function approximator of the table from the
original algorithm. \ac{DQNs} were
introduced in \cite{mnih2015human}, where the algorithm makes use of a
replay memory to store transition tuples of the form $\langle s, a, r,
s' \rangle$. Function $\theta$ is learnt by sampling batches of
transitions from the replay memory and minimizing the following loss
function:
\begin{equation}
\mathcal{L}(\theta)=\sum_{i=1}^{M}\left[\left(y_{i}^{\mathrm{DQN}}-Q(s, a ; \theta)\right)^{2}\right]
\end{equation}
where $y^{\mathrm{DQN}}=r+\gamma \max _{u^{\prime}} Q\left(s^{\prime},
a^{\prime} ; \theta^{-}\right)$ and $\theta^-$ are the parameters of a
target network that are periodically updated (copied) from the
parameters $\theta$ of the training network and kept fixed for a
number of iterations. Note that even when Deep Q-learning cannot guarantee
to find the optimal policy, it is highly effective on large state spaces \cite{mnih2015human}.

{\bf LPOPL.}
\ac{LPOPL} \cite{toro2018teaching} is a reward-tailored
Q-learning function designed for single-agent reinforcement learning
problems with {\em co-safe }$LTL$ specifications. We will employ a
decentralized extension from the deep learning version of this algorithm in our experiments.

{\bf Independent Q-learning.}
Independent Q-learning \cite{tan1993multi} is perhaps the most
spread approach in \ac{MARL}. It basically decomposes a multi-agent
problem into a collection of simultaneous single-agent problems that
share the same environment. Even when this approach does not address
the non-stationarity problem introduced by the changing policies of
the other agents, it nonetheless commonly serves as a strong benchmark
for a range of MAS \cite{tampuu2017multiagent, tang2018hierarchical, rashid2018qmix}.

%% file: EMG.tex
 Reinforcement Learning
is designed to work under Markovian reward models. The main limitation
of this kind of modeling is the Markovian assumption, i.e., rewards
depend only on the last state visited and action taken. Thus we cannot adopt
reward functions that capture conditional temporally extended
properties on vanilla Markov games, such as $LTL$ specifications. In
this section we detail how to handle non-Markovian
specifications with RL in general-purpose game systems.

\subsection{Non-Markovian Reward Games}
In this section we introduce the concept of non-Markovian reward game,
which is similar to Markov games,
but where the reward
depends on the history of states visited and actions taken.
\begin{definition}[NMRG] \label{def:NMRG}
A non-Markovian reward game is a tuple $\mathcal{G}=\langle S, N, A,
P, \overline{R}, \gamma\rangle$ where
\begin{itemize}
\item $S,\, N,\, A,\, P$ and $\gamma$
are defined as for a Markov game in Def.~\ref{markovgame};
\item the reward $\overline{R}$ is a set of real-valued functions over
state-action histories in $(S \times A_1 \times \cdots \times
A_N)^{*}$ (which we will refer as traces), i.e., for every $i \in N$,
$\overline{R}_i :(S \times A_1 \times \cdots \times
A_N)^{*} \rightarrow \mathbb{R}$.
\end{itemize}
\end{definition}

By Def.~\ref{def:NMRG}, NMRGs allow us to define rewards that depend
on a full trace rather than just the last state and action taken. We
are also working with $LTL$ specifications so we embed
them into the NMRG.
\begin{definition} \label{def:NMRG-l}
A NMRG with {\em co-safe} $LTL$ specifications is a tuple $\mathcal{G}=\langle S, N,
A, P, \overline{R}, \mathcal{L}, \Phi, \gamma\rangle$ where:
\begin{itemize}
\item $\langle S, N,
A, P, \overline{R}, \gamma\rangle$ is an NMRG;
\item $\mathcal{L} : S \rightarrow 2^{\mathcal{AP}}$ is
the labelling function; 
\item $\Phi$ is a set of {\em co-safe} $LTL$ specifications.
\end{itemize}
\end{definition}

Note that for the NMRG to be consistent, as it can be intuitively inferred,
$\overline{R_i} \in \overline{R}$ should be correlated to the progression, 
satisfaction or violation of any specification $\varphi_k \in \Phi_i$. 
This correlation must be consistent with the following rule: 
$r_s>r_p>r_v$.Where $\Phi_i$ is the set of specifications associated to 
agent $i$, $r_s$ is the reward granted by $\overline{R_i}$
to the agent $i$ when it satisfies any $\varphi_k \in \Phi_i$, $r_p$ is 
the reward granted when progressing any of those $\varphi_k$
and $r_v$ is the one granted when violating them. By violation of a 
specification we mean that it cannot be fulfilled anymore 
within the current episode.\\

Henceforth, when we mention NMRGs, we will be referring to NMRGs with
$LTL$ specifications. Note that in a NMRG $P$ and $\overline{R}$ are
unknown to agents.
Again, we cannot apply RL algorithms directly to NMRGs since this
would likely lead to non-stationary policies that would not converge. 
Notice that RL is typically based on Markovian Models; thus, 
for most of the RL algorithms to work effectively, we need the reward functions to depend solely on the last state and action taken.

\subsection{Extended Markov Games}\label{subsec:EMG}

In this work we want to learn general purpose non-Markovian $LTL$
specifications in a multi-agent Markovian setting, so as to
exploit \ac{MARL} techniques. This section is devoted to
introduce \acp{EMG} as the mathematical model to achieve this
objective.
\begin{definition} [Problem statement]
Given a NMRG $\mathcal{G}=\langle S, N, A,
P, \overline{R}, \mathcal{L}, \Phi, \gamma\rangle$ and $s_o \in S$ as
the initial state, we want each agent to learn its own non-Markovian
optimal policy $\overline{\pi}^*$ that for each state chooses the
best action leading to satisfy the given {\em co-safe} $LTL$ specification.
\end{definition}

Further, we introduce
a {\em joint non-markovian policy}  $\overline{\Pi}^*$ as the set of the
 agents' individual (non-markovian) policies
 $\{\overline{\pi}_1,\ldots,\pi_N\}$. We represent joint actions, i.e,
 sets of actions simultaneously taken by the agents, as
 $u \in \mathcal{U}$, where $u= \langle a_1,\ldots,a_N \rangle$ and
 $\mathcal{U}= A_1 \times \cdots \times A_N$.

Since we need to use a Markovian model to train our MARL agents we need a notion of equivalence between NMRGs and MGs in a similar fashion to what \cite{bacchus1996rewarding} introduced for single-agent models.

\begin{definition}[Equivalence]  \label{def:equivMG}
A NMRG $\mathcal{G}=\langle S, N, A,
P, \overline{R}, \mathcal{L}, \Phi, \gamma\rangle$ is equivalent to an
extended MG $\mathcal{G'}=\langle S', N, A, P',
R', \mathcal{L}, \Phi, \gamma\rangle$ if there exist two functions
$\tau \, : \, S^{\prime} \rightarrow S$ and $\rho : S \rightarrow
S^{\prime}$ such that:
\begin{enumerate}
        \item For all $s \in S$, $\tau(\rho(s))=s$.

\item For all $s_{1}, s_{2} \in S$ and $s_{1}^{\prime} \in
          S^{\prime}$, if $\operatorname{P}\left(s_{1}, u,
          s_{2}\right)>0$ and $\tau\left(s_{1}^{\prime}\right)=$
          $s_{1},$ there exists a unique $s_{2}^{\prime} \in
          S^{\prime}$ such that
          $\tau\left(s_{2}^{\prime}\right)=s_{2}$ and
          $\operatorname{P}\left(s_{1}^{\prime}, u,
          s_{2}^{\prime}\right)=\operatorname{P}\left(s_{1}, u,
          s_{2}\right)$.
          
        \item For any feasible trajectory $\left\langle s_{0},
    u_{1}, \ldots, s_{l-1}, u_{l}, s_l\right\rangle$ of $\mathcal{G}$
    and $\left\langle s_{0}^{\prime}, u_{1}, \ldots, s_{l-1}^{\prime},
    u_{l}, s'_l\right\rangle$ of $\mathcal{G}^{\prime},$ such that
    $\tau\left(s_{t}^{\prime}\right)=s_{t}$ and
    $\rho\left(s_{0}\right)=s_{0,}^{\prime}$ we have
    $\overline{R}\left(\left\langle s_{0}, u_{1}, \ldots, s_{l-1},
    u_{l}, s_l\right\rangle\right)=$ $R^{\prime}\left(\left\langle
    s_{0}^{\prime}, u_{1}, \ldots, s_{l-1}^{\prime}, u_{l},
    s'_l\right\rangle\right)$.
\end{enumerate}
where by {\em feasible trajectory} we refer to any trajectory that is
consistent within the transitions in the game.
\end{definition}

The crucial points in Def.~\ref{def:equivMG} are clauses (2) and (3),
which assert the equivalence of the two models (with respect to the initial
states) in both their dynamics and reward structure.
In particular, clause (2) ensures that for any given trajectory in
$\mathcal{G}$
$$
s_{0}, u_{1} \stackrel{P(s_{0}, u_1, s_{1})}{\longrightarrow} s_{1}, u_2 \cdots s_{l-1}, u_{l} \stackrel{P(s_{l-1}, u_l, s_{l})}{\longrightarrow} s_{l}
$$
and for any extended state $s_0'$ with base state $s_0$, i.e., $\tau(s'_0)=s_0$, there is a trajectory in $\mathcal{G'}$ of similar structure: 
$$
s'_{0}, u_{1} \stackrel{P'(s_{1}, u_1, s_{2})}{\longrightarrow} s'_{1}, u_2 \cdots s'_{l-1}, u_{l} \stackrel{P'(s_{l-1}, u_l, s_{l})}{\longrightarrow} s_{l}
$$
where $\tau(s_t')=s_t$ for all $t$ in the trajectory. In this case,
trajectories $\langle s_0,u_1,\ldots,s_l\rangle$ and $\langle
s'_0,u_1,\ldots,s'_l\rangle$ are called
\textit{weakly correspondent}. Clause (3) on the other hand, imposes requirements on the individual rewards assigned to each of the agents in the
extended game $\mathcal{G'}$. If we have two weakly correspondent
trajectories $\langle s_0,u_1,\ldots,s_l\rangle$ and $\langle
s'_0,u_1,\ldots,s'_l\rangle$ such that
$\rho\left(s_{0}\right)=s'_{0}$, we say that these trajectories
are \textit{strongly correspondent}. This also means that their
relationship is one-to-one, that is, each trajectory in $\mathcal{G}$
has a unique strongly correspondent in $\mathcal{G'}$, and that a
trajectory in $\mathcal{G'}$ has a unique strongly correspondent if
its first state
is an initial state. Thus, clause (3) requires that the functions in
$R'$ assign individual rewards to extended states in such a manner
that strongly corresponding trajectories
receive the same reward.\\


In a similar fashion to what \cite{brafman2018ltlf} did extending MDPs
we introduce the following definition:
\begin{definition} [Equivalent MG] \label{equivalentMG}
Given an NMRG $\mathcal{G}=\langle S, N, A,
P, \overline{R}, \mathcal{L}, \Phi, \gamma\rangle$ with
$\Phi= \{\varphi_1,\ldots,\varphi_M\}$, let $\mathcal{A}_k
= \left\langle 2^{\mathcal{AP}}, Q_{k}, q_{k
0}, \delta_{k}, F_{k} \right\rangle$ be the DFA corresponding to $\varphi_k$.
We define the {\em equivalent
Markov game} $\mathcal{G'}=\langle S', N, A, P',
R', \mathcal{L}, \Phi, \gamma\rangle$ such that
%
%
\begin{itemize}
    \item $S^{\prime}=Q_{1} \times \cdots \times Q_{M} \times S$ is the set of states;
    \item $\mathcal{U}, \, \Phi, \, \mathcal{L}$ and $\delta$ are defined as in the NMRG. 
    \item $P^{\prime} : S^{\prime} \times \mathcal{U} \times S^{\prime} \rightarrow[0,1]$ is defined as follows:
    {\footnotesize
    $$
    \begin{array}{cl}{\operatorname{P}^{\prime}\left(\vec{q}, s, u, \vec{q'} , s^{\prime}\right)=} {\left\{\begin{array}{ll}{\operatorname{P}\left(s, u, s^{\prime}\right)} & {\text { if } \forall i : \delta_{i}\left(q_{i}, s^{\prime}\right)=q_{i}^{\prime}} \\ {0} & {\text { otherwise }}\end{array}\right.}\end{array}
    $$}
    \item $R^{\prime} : S^{\prime} \times \mathcal{U} \times S^{\prime} \rightarrow \mathbb{R}$ is defined as:
    {\footnotesize
    $$
    R^{\prime}\left(\vec{q} , s, u, \vec{q'} , s^{\prime}\right)=\sum_{i : q_{i}^{\prime} \in F_{i}} r_{i}
    $$}
\end{itemize}
\end{definition} 

Intuitively, by Def.~\ref{equivalentMG} the extended state space in
$\mathcal{G'}$ is a product of the states in the original NMRG and the
automata for the various $LTL$ formulas. Notice that, since we assume a
fully observable environment, it is important that the new state space
perceived by every agent is composed by all automata.
Given a join action $u$, the $s$-component in the extended state
progresses according to the original dynamics in the NMRG, while the
transition function of the corresponding automaton marks the progress
of the other components.


\begin{thm} 
The NMRG $\mathcal{G}=\langle S, N, A, P, \overline{R}, \mathcal{L}, \Phi, \gamma\rangle$ is equivalent to the MG $\mathcal{G'}=\langle S', N, A, P', R', \mathcal{L}, \Phi, \gamma\rangle$ defined above.
\end{thm}

\begin{proof}
Recalling that every $s'\in S'$ has the form $(q_1,\ldots,q_M,s)$, we
define $\delta\left(q_{1}, \ldots, q_{M}, s\right)=s$ and
$\sigma(s)=\left(q_{10}, \dots, q_{M 0}, s\right)$. We have that
$\delta(\sigma(s))=s$. Condition 2 of Def.~\ref{def:equivMG} can be
easily verified by inspection. For condition 3, let's consider a trace
$\eta=\left\langle s_{0}, u_{1}, \ldots, s_{n-1}, u_{n},
s_n\right\rangle$. We use $\sigma$ to obtain $s'_0=\sigma(s_0)$ and
given $s_t$ we define $s'_t$ (for 1$\leq t < n$) to be the unique
state ($q_{1 t}, \ldots, q_{M t}, s_{t}$) such that $q_{j
t}=\delta\left(q_{j t-1}, u_{t}\right)$ for all $1\leq j \leq
M$. Thus, we now have a corresponding trace in $\mathcal{G'}$, i.e.,
$\eta^{\prime}=\left\langle s_{0}^{\prime}, u_{1},
s_{1}^{\prime} \ldots, s_{n-1}^{\prime}, u_n, s_n\right\rangle$, which
is the only feasible trajectory in $\mathcal{G'}$ that satisfies
Condition 3. The reward at $\eta$ depends solely on whether or not
each specification $\varphi$ is satisfied by $\eta$. Nevertheless, by
construction of the DFA $\mathcal{A}_{k}$ and the
transition function $P'$, we have that $\eta \vDash \varphi_{k}$ iff
$s_{n-1}^{\prime}=\left(q_{1}, \ldots, q_{M}, s_{n}^{\prime}\right)$
and $q_{k} \in F_{k}$
\end{proof}

Given  a joint Markovian policy $\Pi'$ for $\mathcal{G}'$, a
policy $\overline{\Pi}$ on $\mathcal{G}$ that guarantees the same rewards
can be easily found. To this end, consider a trace $\sigma
= \left\langle s_{0}, u_{1}, s_{1}, \dots, s_{l-1}, u_{l},
s_l\right\rangle$ in $\mathcal{G}$ and let $q_i$ be the state of
$A_{i}$ on the input $\sigma$. We define the (non-Markovian)
joint policy $\overline{\Pi}$ equivalent to $\Pi'$ as
$\overline{\Pi}(\sigma)=\Pi'\left(q_{1}, \ldots, q_{M},
s_{l}\right)$. Given this and Def.~\ref{def:equivMG}, we
can define optimal policies for $\mathcal{G}$ by solving
$\mathcal{G'}$ instead. Extending \cite{bacchus1996rewarding} to a
multi-agent setting we obtain:
\begin{thm} Given a NMRG $\mathcal{G}$ , let $\Pi'$ be a set of optimal policies for an equivalent MG $\mathcal{G}'$. Then, the policy $\overline{\Pi}$ for $\mathcal{G}$ that is equivalent to $\Pi'$ is optimal for $\mathcal{G}$.
\end{thm}

It can be then deduced that multi-agent RL techniques can be directly
applied to $\mathcal{G}'$, so that we can obtain an optimal
joint policy $\Pi'^*$ for $\mathcal{G}'$. Therefore, an optimal joint
policy for the NMRG $\mathcal{G}$ can be learnt from $\mathcal{G}'$.
\begin{rem}
Note that none of these game structures is explicitly known by the
learning agents, in the sense that the agents are never given the
whole structure as input, but rather they just get to know what they
can observe through their interactions. Thus, in practice
the above transformations are never done explicitly. Instead, the agents
 learn by assuming that the underlying model is
$\mathcal{G}'$.
\end{rem}

Since the state space of $\mathcal{G}'$ is the product of the state
spaces of $\mathcal{G}$ and of the automata $A_{\varphi}$, the rewards in
$R'$ are Markovian, i.e., the state representation of the
temporal formulae in $\Phi$
are compiled into the states of $\mathcal{G}'$. As
in \cite{de2018reinforcement} for single agents, and given Cor.~\ref{cor1}, 
we can state the following result in the multi-agent case:
\begin{cor} \label{cor2}
RL for {\em co-safe} $LTL$ rewards $\varphi$ over an NMRG $\mathcal{G}=\langle S,
N, A, P, \overline{R}, \mathcal{L}, \Phi, \gamma\rangle$ can be
reduced to RL over the EMG $\mathcal{G}'$ in Def.~\ref{equivalentMG}.
\end{cor}

\begin{rem} \label{rem:temporalLanguages}
We observe that Cor.~\ref{cor2} holds for
specifications in $\Phi$ expressed in
any temporal language as long as there is a way to compute a DFA for
each of the specifications.
\end{rem}
%
This means that we can employ $LTL$ as well as the variants we
presented in Section \ref{subsec:LTL}, whose queries can be directly
transform into DFAs, or even more expressive languages such as Linear
Dynamic Logic over finite traces $LDL_f$.
\cite{brafman2018ltlf}.
\begin{example} \label{examp:NMRG}
 Consider the specification of making shears we made in our last
 example. If we want to reward our agents for making shears given the
state space we first defined, we would need to check the state-action
 history of the agents when using the workbench to reward them
 only if they previously got wood and iron. Since rewards depend on the
 trace followed by the agents, our model would not be a Markov game
 anymore but a NMRG. In order to model our problem as an Extended Markov Game,
 we transform the $LTL$ specification of making shears into a DFA whose states would be given
 to the agents as an extension of the original observation that they perceive from the environment. The agents would then begin the episodes by perceiving an observation of the map extended by the initial state of the automaton, that represents the whole specification of making shears to be fulfilled. Once the agents have progressed the specification, which in this case means that they got iron or wood, the automaton will transit to a new state that represents the remainder of the specification to be fulfilled. The agents will notice a change in their perception of the environment because of this new state in the DFA. Hence, the agents will perceive the process as Markovian. 
\end{example}

\begin{rem}
Note that the dimensionality of the extended state space is not dependant on the number of specifications, as it extends Markov Games only with the representation of the current DFA state. Thus, following the methodology presented in our Section \ref{sec:experiments}, we only need to extend the state space with a single feature. When working with multiple specifications, these can still be represented with a single DFA. 
\end{rem}

%% file: Deep_MARL.tex
To empirically support our theoretical results on EMGs, this section
is devoted to introduce two simple decentralized extensions of
single-agent \ac{DRL} algorithms designed to work with $LTL$
specifications. The first approach is based on extending with 
temporal logic specifications a popular baseline in
MARL called I-DQN \cite{tampuu2015multiagent}, while the second is a
multi-agent extension of LPOPL that we referred in Sec~\ref{sec:intro}. 
The extended algorithms described below are employed in the experiments presented 
in Sec.~\ref{sec:experiments}.

    
\subsection{I-DQN with {\em co-safe }$LTL$ Goals}

The first algorithm is based on DQNs defined in Sec.~\ref{sec:backg}. Q-values
for a given agent $i$ can be defined in the multi-agent
context \cite{foerster2017stabilising} by fixing the policy of all the
other agents. Let $Q^{\pi_i}_i$ be the Q-value for a given state $s$
and action $a$ for agent $i$, defined as follows:
\begin{equation}\label{eq:multiQ}
\begin{split}
    Q^{\pi_i}_i( s_{0}, a_i  \;|\; \Pi_{-i})=\mathbb{E}_{\pi_i}\left[\sum_{t=0}^{T} \gamma^{t} r_t \;|\; s_0,a_i, \Pi_{-i}\right]
\end{split}
\end{equation}
where $\Pi_{-i}$ are the policies of all agents other then $i$, and
$T$ is the length of the trace.

Intuitively, Equation~(\ref{eq:multiQ}) states that the quality of an
action-state pair, given a policy for an agent $i$, can be expressed as
in the single-agent case given the policies of all other agents.
This allows Independent Q-learning \cite{tan1993multi} to train
multiple agents in a decentralized fashion. Here we consider a deep
learning variant of this algorithm (see,
e.g., \cite{tampuu2015multiagent}), where each agent is trained with
an independent DQN. However, in our case, we adopt a decentralized
version of an algorithm that uses $LTL$ specifications and $LTL$
progression instead of classical reward functions (see,
e.g., \cite{littman2017environment}). Hereafter
we refer to this algorithm
as I-DQN-l. Since we are working with sets of specifications, I-DQN-l also
incorporates a curriculum learner that selects the specification that
the agents will have to fulfill in the next episode. We selected a simple
method that assumes the specifications in $\Phi$ follow some order
$\varphi_0,\ldots,\varphi_{M-1}$ where $|\Phi|= M$, and the curriculum
learner just selects the next specification following the given order.

\subsection{I-LPOPL}
Our second algorithm for experimental evaluations
is a multi-agent extension of the LPOPL \cite{toro2018teaching}, which
is designed specifically to take advantage of {\em co-safe} $LTL$ specifications in
order to boost agents learning performance. Similarly to what we
described above for I-DQN-l, we combine the version of LPOPL that
employs DQNs as function approximators with Independent Q-learning to
obtain a centralized-decentralized algorithm we call Independent-LPOPL
(I-LPOPL). Notice that, in MARL, centralized components refer to those
that are shared among all the agents, while decentralized are those
components that run locally for each agent. Similarly to the single
agent algorithm, I-LPOPL is based on four main components:
\begin{itemize}

\item \textbf{A centralized task extractor.} We use the original task extractor from LPOPL in a centralized manner. This component receives as input the global set of specifications $\Phi$, and returns an extended set $\Phi^+$ of sub-specifications, or \textit{tasks}, which contains each unique $LTL$ task that the original set of specifications can be divided in, including the original set  $\Phi$ of specifications. For instance, given $\Phi=\{\diamond(b \wedge \diamond c), \diamond(d \wedge \diamond a)\}$, after a first progression the specifications become $\diamond c$ and $\diamond a$. Thus, here the task extractor would generate the extended set $\Phi^+=\{\diamond(b \wedge \diamond c), \diamond(d \wedge \diamond c), \diamond c , \diamond a\}$.

\item \textbf{Decentralized LPOPL behavior policies.} In I-LPOPL each agent has its own set of networks, as each agent has a different DQN working as a Q-value function approximator for each extracted task. If $\varphi$ is the specification to solve in the current episode and $\varphi^+$ the result of progressing $\varphi$ through the history of the episode so far, then the behavior policy is $\epsilon$-greedy on the I-DQN-L whose goal is to satisfy $\varphi^+$, i.e, the behavior policy for each agent would be the one optimized to solve the current goal task observed by the agent in the extended state of the EMG.

    \item \textbf{Decentralized Q-value function updates.} In each
    transition, each agent updates its DQNs independently, that is, at
    each time step, the sequence $\langle s'_t,u_t,s'_{t+1},
    r_{t+1}\rangle$ is stored in the replay buffer of each DQN, where
    $r$ is given accordingly to the task each DQN is
    learning. However, in order to allow agents to develop different
    and potentially complementary individual policies, the update
    process for each agent is independent from the others'.

\item \textbf{A centralized curriculum learner}: At the
    beginning of every iteration of I-LPOPL a global curriculum
    learning method is used to select the next specification to learn,
    in the same fashion as in I-DQN-l.
\end{itemize}

A training iteration of I-LPOPL proceeds as follows: the task
extractor receives the set of specifications and creates the sets of
tasks. Then, an independent DQN is initialized by each agent for each
of these tasks. Once done, the curriculum learner selects the next
specification to be solved by the agents in the environment. The
agents start interacting with the environment, and once there are
enough interactions to start learning, the Q-value function update algorithm 
is called after each transition. When the
curriculum learner triggers the end of the learning process for the
current task, a new one is selected. This process is repeated through
all specifications.

%% file: experiments.tex
In this section, we
demonstrate that multiple agents can
concurrently learn policies that satisfy multiple non-Markovian
specifications.
Notice that the source code of our experiments is publicly available
at \url{https://github.com/bgLeon/EMG}.

\subsection{Experimental Setup}
This section is devoted to explain the testing environment, 
and the features used in our experiments. Our goal here 
is first to prove that multiple RL agents can be trained to solve 
multiple non-Markovian specifications,
and second that agents can develop multi-agent policies, in this
case collaborative policies, to solve the given specifications.


\textbf{The Minecraft-like world.} The environment where we test our algorithms is the Minecraft-like grid map described in Examples~\ref{examp:MG} and~\ref{examp:NMRG}. We chose this
environment as it is similarly used as testing set-up
in \cite{de2018reinforcement,toro2018teaching,
andreas2017modular}. Specifically, our environment and test setting is
a multi-agent version of the one presented in \cite{toro2018teaching},
where we introduce two agents
that need to fulfill different sets of {\em co-safe} $LTL$ specifications.
Two agents cannot be in the same place at the same
time, which means that agents can obstruct each other. We model the
environment as a discounted reward problem including the relevant
DFA. That is, given the current specification $\varphi$ to be solved and its
associated DFA $\mathcal{A}_{\varphi}$, a reward of -1 is given each time the
DFA remains in the same state, 0 if the DFA transitions to a new
intermediate state (i.e, the agents have progressed in the
specification), and +1 if the DFA reaches a terminal state (i.e, the
agents fulfilled the specification).

\textbf{Features.}
The two algorithms  consider
the same features, actions, network architecture, and optimizer. The
input features are contained in a vector that also registers the
distance of every object to the agent receiving the input, as
in \cite{andreas2017modular, toro2018teaching}. The DQNs code is
based on the implementation from OpenAI Baselines \cite{baselines}.
Specifically, we use a feedforward network with 2 hidden layers of 64
neurons using ReLu \cite{glorot2011deep} as activation function. The
networks are train using Adam optimization \cite{kingma2014adam} with
a learning rate of $5x10^{-4}$, sampling with a batch size of 32
transitions over a replay buffer with a sample size of size 25,000 and updating
the target networks every 100 steps.  The discount factor for both
DQN-l and I-DQN-l is 0.98, which works better with them since their
networks are trained to solve a full specification; while the discount
factor in both LPOPL and I-LPOPL is 0.9, which is a better choice for
these algorithms, since
each DQN is focused on
solving just a sub-specification or short-term task.

\subsection{Specification Sets}
This section introduces and explains the different sets of
specifications to be solved by the agents, also based on a multi-agent
extension of \cite{toro2018teaching,andreas2017modular}:

\begin{figure}[t]
    \centering \includegraphics[width=0.5\textwidth, height=4.35cm]{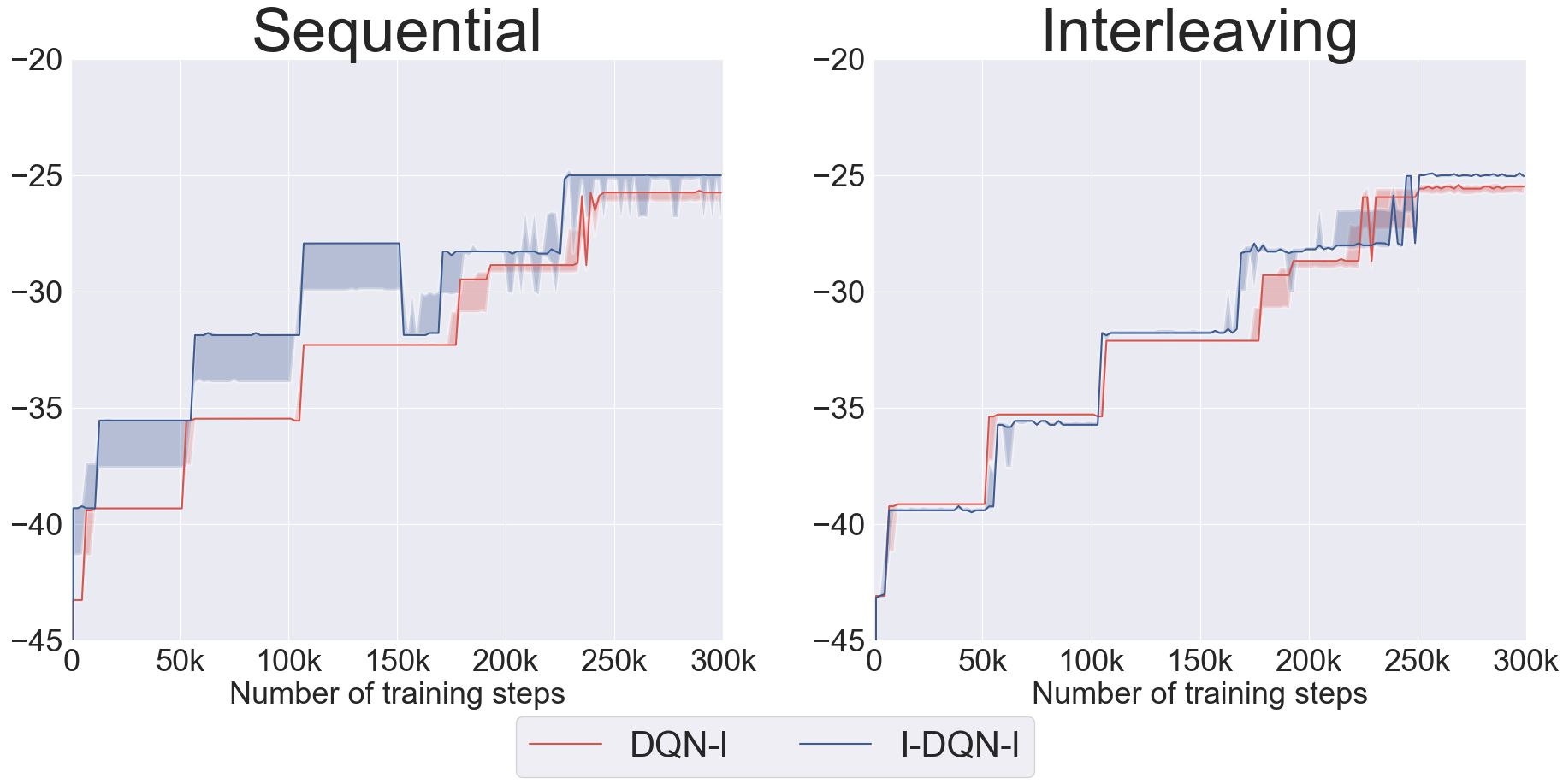} \caption{Rewards 
    obtained by DQN-l and I-DQN-l greedy policies evaluated on the whole
    set of sequence-based specifications (left) and on the whole set
    of unordered specifications (right). The algorithms were trained
    for 50k training steps per specification of the
    set.}  \label{fig:I-DQN}
\end{figure}
\begin{itemize}
    \item \textbf{Specifications as sequences of tasks.} The first set
    contains 10 specifications from \cite{andreas2017modular} which are
    a simple sequence of properties to be satisfied
by the two agents. For instance, the specification \textit{make\_shears} defined in Example~\ref{examp:MG}, 
can be translated to a sequential $LTL$ specification as
$ \diamond(\textit{got\_wood} \wedge \diamond\left(\textit{used\_workbench} 
\wedge \diamond(\textit{got\_iron} \wedge \diamond \textit{used\_workbench})\right)$. While a single
agent can fulfill the whole sequence, an intuitively better policy
would be one where one agent goes to get wood and the other waits at the
workbench until the first agent collects the wood, then the agent
closer to iron moves towards it and the other agent waits at the
workbench.


\item \textbf{Specifications as interleaving tasks.} This
set, proposed in \cite{toro2018teaching}, differs from the previous one
as specifications are no longer a fixed sequence, but rather
unnecessary ordering over parts of the specifications
are removed. The $LTL$ specification that we used in Example~\ref{examp:MG}
for making shears belongs to this set.
\end{itemize}

Note that, while we extend the experiments
in \cite{toro2018teaching}, our single-agent version of the algorithms
do not match exactly the setting used in the reference.
Specifically, we changed the discount factors, the learning
rate, and the reward functions
in order to optimize the learning times of the algorithms, a much
needed step when doing decentralized multi-agent experiments
due to
scalability with the number of agents.

\subsection{Experimental Results}

In order to evaluate our approach, we ran the multi-agent
algorithms introduced in Section \ref{sec:MARL} in a 
multi-agent version of one of the random maps from \cite{toro2018teaching}, 
and tested them against the correspondent single-agent algorithm in the unaltered map
from \cite{toro2018teaching}. In both versions, the
agents had a time-limit of 300 steps to solve the specification
assigned to an episode. Both algorithms were trained for a given
number of training steps per specification (see Figures~\ref{fig:I-DQN}
and \ref{fig:I_LPOPL}). Once the limit was reached, the curriculum
learner selected the next task to be learnt. Every 1k steps we
measured the performance of the greedy policies developed by the
algorithms on the whole set of specifications. Figures~\ref{fig:I-DQN}
and \ref{fig:I_LPOPL} shows the mean rewards obtained by these
policies, and the shadowed areas are the $25^{th}$ and $75^{th}$
percentiles obtained across 3 independent runs per algorithm per
set of tasks.
\begin{figure}[t]
    \centering \includegraphics[width=0.5\textwidth]{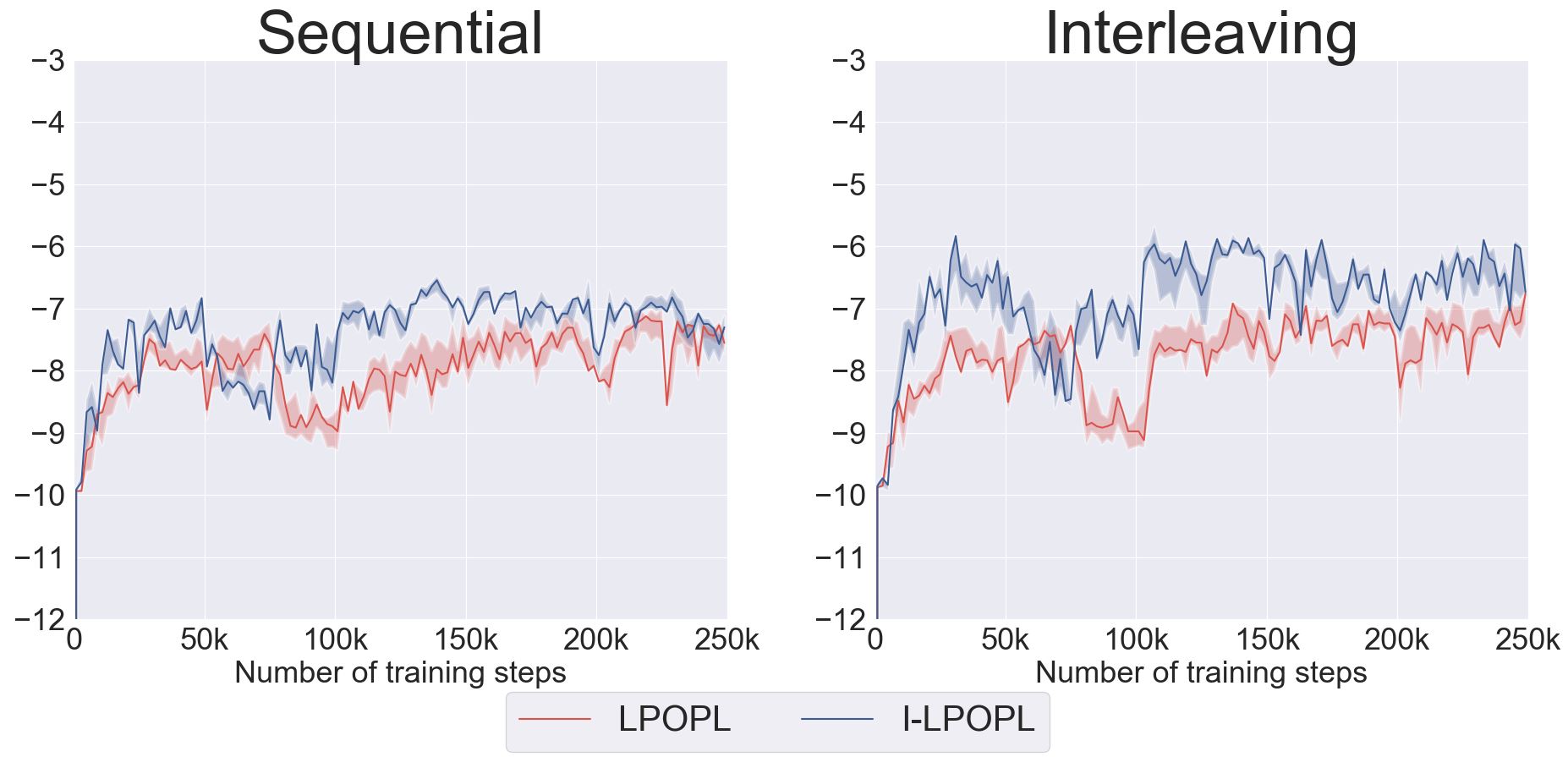} 
    \caption{Rewards obtained by LPOPL and I-LPOPL greedy policies evaluated on the whole
    set of sequence-based specifications (left) and on the whole set
    of unordered specifications (right). The algorithms were trained
    for 25k training steps per specification of the
    set.}  \label{fig:I_LPOPL}
\end{figure}
Figure \ref{fig:I-DQN} compares the performance of I-DQN-l and
DQN-l. These algorithms have a different policy network
for each specification, where each network is updated only when the
curriculum learner have selected the correspondent specification. This is
why we observe steeped reward figures in the plot. I-DQN-l achieves
higher rewards that DQN-l in the single-agent map thanks to the
collaborative strategies that the agents develop in the multi-agent
setting, i.e, if the specification requires to collect an item and
then use a tool, typically one agent goes for the item while the other
waits with the tool to use it as soon as the first agent has
finished.

Figure \ref{fig:I_LPOPL} show the performance of I-LPOPL and
LPOPL. Again, the higher rewards obtained by I-LPOPL over LPOPL are
due to agents collaborating in the multi-agent setting. These
algorithms do not show steeped reward figures because at each training
step every network is updated accordingly to its assigned
specification. We note, however, that this policy update system is also
causing some drops in the performance of the multi-agent algorithm
such as after the 50k step. When the behavior policy is used for a
given task, the agents learn that the best policy is for the farthest
agent to move towards the next goal in the specification. However,
this is not the case for the other policies that are updated without
being used as behavior policy. For instance,
the agents learn to solve the initial specification first with a collaborative
policy since when the two agents going for the same object also means that
they are obstructing each other's path. However, when these policy
networks are later updated without being used as a behavior policy, 
the Q-values are updated without encountering the other agent
obstruction anymore. Thus, the agent will in some cases forget the
collaborative policies and select single-agent ones, causing drops in
performance as those in Figure~\ref{fig:I_LPOPL}.

Finally, we highlight that while I-LPOPL learns faster that I-DQN-l in terms of
training steps, I-LPOPL demands higher computing resources and needs
longer time to train. The reason for this is that in I-DQN-l we have a
DQN per agent per specification, while in I-LPOPL-l we have a DQN per
agent per sub-specification. In our experiments, we had 10
specifications in both set of tasks, 27 sub-specifications in the
sequential set and 34 in the interleaving set. The table below shows
the computing time for each algorithm in the experiments. The
experiments were run on a laptop with a GeForce RTX 2080 as GPU, 
32 GB of RAM and a i7 7700 HQ as CPU, computing the three independent runs in parallel.
\begin{center}
\small
\begin{tabular}{ |P{1.5cm}||P{1.5cm}|P{1.5cm}| } 
 \hline
 \multicolumn{3}{|c|}{Computing time (hours)}  \\
 \hline
 Algorithm & Sequential  & Interleaving \\
 \hline
 DQN-l   & 1.20 & 1.30 \\
 I-DQN-l  & 1.57 & 2.15 \\
 LPOPL     & 5.13 & 6.5 \\
 I-LPOPL    & 11.32 & 14.28\\
 \hline
\end{tabular}
\end{center}
\hfill\break

%% file: conclusions.tex
We introduced Extended Markov Games as a  mathematical model
for multi-agent reinforcement learning,
to learn policies that satisfy multiple (non-Markovian) $LTL$
specifications in
multi-agent systems. Our formal definitions actually infer that any temporal logic
can be used to express the specifications as long as they can be converted to a DFA.
Our experimental results in collaborative games demonstrate that the proposed framework
allows RL agents to learn and develop multi-agent strategies to
satisfy different set of specifications.

Based on the present contribution, future directions of research
include games with partial observability, where occlusion can refer
both to the agent's ability to see the whole map, as well as to the ability to
detect other agent's state. As regards the algorithm component, future lines
include merging temporal logic with native MARL algorithms that scale
better on the number of agents.